\documentclass[]{fairmeta} 

\usepackage{amsmath,amsfonts,bm}

\def\eqref#1{equation~\ref{#1}}

\def\1{\bm{1}}

\DeclareMathAlphabet{\mathsfit}{\encodingdefault}{\sfdefault}{m}{sl}
\SetMathAlphabet{\mathsfit}{bold}{\encodingdefault}{\sfdefault}{bx}{n}

\newcommand{\E}{\mathbb{E}}

\newcommand{\R}{\mathbb{R}}

\usepackage[inline]{enumitem}
\usepackage[utf8]{inputenc}
\usepackage[T1]{fontenc}

\usepackage{amsfonts}
\usepackage{amsmath}
\usepackage{amssymb}
\usepackage{amsthm}

\usepackage{xspace}
\usepackage{multirow}
\usepackage{algorithm}
\usepackage{algpseudocode}

\DeclareRobustCommand{\grpo}{\textnormal{\textsc{GRPO}}\xspace}
\DeclareRobustCommand{\rloo}{\textnormal{\textsc{RLOO}}\xspace}

\title{SCA: Spatial Credit Assignment for Reinforcement Learning of GUI Agents}

\author[1]{Shengtian Yang}
\author[1]{Ziyu Xiong}
\author[1]{Kaibing Yang}
\author[1]{Guangfeng Cai}
\author[2]{Yewen Li}
\author[3]{Peng Jiang}
\author[3]{Gai Kun}
\author[2]{Qingpeng Cai}
\author[1,\dagger]{Lei Feng}

\affiliation[1]{Southeast University, Nanjing, China}
\affiliation[2]{Kuaishou Technology, Beijing, China}
\affiliation[3]{Unaffiliated}

\contribution[\dagger]{Corresponding author}

\abstract{
GUI agents automate tasks on digital devices by grounding language instructions in visual interfaces. Existing group-relative reinforcement learning improves GUI action prediction by comparing the rewards of multiple responses sampled from the same GUI state. However, binary evaluation treats spatially different failed clicks as identical and provides no relative signal when all sampled clicks fail. To address these limitations, we propose Spatial Credit Assignment (SCA), which uses the screen coordinates of sampled clicks to refine group-relative credit. Specifically, SCA predicts each held-out response's reward from the other responses in groups containing both successes and failures, then uses the prediction residual to adjust credit. When all sampled clicks fail, SCA instead orders them by distance to the annotated target. These spatial references are used only to construct the training update; the deployed policy remains unchanged. We evaluate whether this correction improves the policy update itself by comparing its error and directional alignment with the exact return gradient in a controlled synthetic study. Across GUI grounding and offline action-prediction benchmarks, SCA improves grounding across professional domains and achieves the strongest results among reinforcement-fine-tuned models on most action-prediction metrics, with consistent gains across the reported GUI suites.
}

\correspondence{\email{yangshengtian@kuaishou.com}}

\usepackage[utf8]{inputenc}
\usepackage[T1]{fontenc}
\usepackage{hyperref}
\hypersetup{hypertexnames=false,hidelinks}
\usepackage{url}
\usepackage{booktabs}
\usepackage{amsfonts}
\usepackage{amsmath}
\usepackage{amssymb}
\usepackage{amsthm}
\usepackage{nicefrac}
\usepackage{microtype}
\usepackage{enumitem}
\usepackage{wrapfig}
\usepackage[table]{xcolor}
\usepackage{graphicx}
\usepackage{xspace}
\usepackage{multirow}
\usepackage{array}
\usepackage{algorithm}
\usepackage{algpseudocode}
\usepackage{float}
\usepackage{tikz}
\usepackage[most]{tcolorbox}
\usepackage{natbib}
\usepackage{placeins}

\definecolor{coreblue}{RGB}{35,96,170}
\definecolor{proxorange}{RGB}{215,118,42}
\definecolor{basegray}{RGB}{150,156,163}
\definecolor{guirorange}{RGB}{224,126,47}
\definecolor{scablue}{RGB}{39,111,177}
\newif\ifscabenchmarktables
\scabenchmarktablesfalse
\newtheorem{proposition}{Proposition}

\providecommand{\sca}{\textup{\textsc{SCA}}\xspace}
\providecommand{\scaresidual}{\mbox{\textup{\textsc{SCA-Residual}}}\xspace}

\providecommand{\grpo}{\textup{\textsc{GRPO}}\xspace}
\providecommand{\rloo}{\textup{\textsc{RLOO}}\xspace}

\providecommand{\E}{\mathbb{E}}
\providecommand{\R}{\mathbb{R}}
\providecommand{\Ind}{\mathbb{I}}
\providecommand{\sg}{\operatorname{sg}}
\providecommand{\clip}{\operatorname{clip}}

\newtcolorbox{scaprompt}[1]{
  enhanced, breakable, colback=gray!4, colframe=black,
  colbacktitle=black, coltitle=white, title={#1},
  fonttitle=\bfseries, fontupper=\footnotesize\ttfamily\raggedright,
  arc=2mm, boxrule=0.7pt, left=2.5mm, right=2.5mm,
  top=1.2mm, bottom=1.2mm, before skip=5pt, after skip=6pt
}
\newtcolorbox{scacase}[1]{
  enhanced, breakable, colback=white, colframe=black,
  colbacktitle=black, coltitle=white, title={#1},
  fonttitle=\bfseries, fontupper=\small,
  arc=2mm, boxrule=0.7pt, left=3mm, right=3mm,
  top=2mm, bottom=2mm, before skip=7pt, after skip=7pt,
  segmentation style={solid,black!55,line width=0.5pt}
}

\newcommand{\SCAWebTextPM}{90.5\pm0.30}

\newcommand{\SCAWebIconPM}{73.5\pm0.42}

\newcommand{\SCADesktopTextPM}{94.8\pm0.28}

\newcommand{\SCADesktopIconPM}{66.5\pm0.45}

\newcommand{\SCAGAGRPM}{89.0\pm0.42}
\newcommand{\DeltaGAGR}{1.58}

\newcommand{\SCAGASRPM}{78.5\pm0.50}

\newcommand{\SCAOWGRPM}{77.0\pm0.45}
\newcommand{\DeltaOWGR}{1.90}

\newcommand{\SCAOWSRPM}{77.1\pm0.47}

\newcommand{\SCAODGRPM}{80.5\pm0.41}
\newcommand{\DeltaODGR}{2.13}

\newcommand{\SCAODSRPM}{80.6\pm0.43}

\newcommand{\SCALowOverallPM}{82.0\pm0.36}

\newcommand{\GUIOdysseySR}{64.41}
\newcommand{\UIOdysseySR}{66.44}

\newcommand{\SCAOdysseySRPM}{66.0\pm0.52}
\newcommand{\DeltaOdysseySR}{1.59}
\newcommand{\SCAProMean}{26.1}

\newcommand{\SCAScreenHeight}{2.033125}
\newcommand{\SCALowHeight}{2.050000}
\newcommand{\SCAOdyHeight}{1.650000}
\newcommand{\GUIProMean}{25.2}

\newcommand{\GUIScreenHeight}{2.001875}
\newcommand{\GUILowHeight}{2.022000}
\newcommand{\GUIOdyHeight}{1.610250}

\newcommand{\QwenScreenHeight}{1.308750}
\newcommand{\QwenLowHeight}{1.391250}
\newcommand{\QwenOdyHeight}{1.483000}
\newcommand{\SCAProText}{42.00}
\newcommand{\GUIProText}{40.77}
\newcommand{\DeltaProText}{1.23}
\newcommand{\SCAProIcon}{10.17}
\newcommand{\GUIProIcon}{9.68}
\newcommand{\DeltaProIcon}{0.48}

\newcommand{\DeltaGroundMin}{0.3}
\newcommand{\DeltaGroundMax}{1.7}
\newcommand{\DeltaOdyUI}{0.44}
\newcommand{\DomainComparisonRows}{
Dev & 19.30 & 20.05 & +0.75 \\
CAD & 17.10 & 17.85 & +0.75 \\
Creative & 23.25 & 23.95 & +0.70 \\
Scientific & 39.55 & 40.50 & +0.95 \\
Office & 35.30 & 36.40 & +1.10 \\
OS & 16.85 & 17.75 & +0.90 \\
}
\newcommand{\QwenProCoordinates}{(90:0.352000)--(30:0.298000)--(-30:0.494000)--(-90:0.892000)--(-150:0.562000)--(150:0.322000)}
\newcommand{\QwenActionCoordinates}{(90:1.122000)--(50:1.285600)--(10:1.112200)--(-30:1.012600)--(-70:0.937800)--(-110:0.940400)--(-150:1.139000)--(170:0.959400)--(130:0.937800)}
\newcommand{\GUIProCoordinates}{(90:0.772000)--(30:0.684000)--(-30:0.930000)--(-90:1.582000)--(-150:1.412000)--(150:0.674000)}
\newcommand{\GUIActionCoordinates}{(90:1.797200)--(50:1.748400)--(10:1.526200)--(-30:1.771600)--(-70:1.502000)--(-110:1.501600)--(-150:1.837200)--(170:1.567400)--(130:1.566200)}
\newcommand{\SCAProCoordinates}{(90:0.802000)--(30:0.714000)--(-30:0.958000)--(-90:1.620000)--(-150:1.456000)--(150:0.710000)}
\newcommand{\SCAActionCoordinates}{(90:1.824000)--(50:1.780000)--(10:1.570000)--(-30:1.802000)--(-70:1.540000)--(-110:1.542000)--(-150:1.860000)--(170:1.610000)--(130:1.612000)}

\begin{document}
\raggedbottom
\emergencystretch=4em
\maketitle
\section{Introduction}
\label{sec:intro}
GUI agents map screenshots and natural-language instructions to executable actions, such as clicking controls, selecting menu items, and entering text~\citep{seeclick,gui-r1,ui-r1,osatlas,uground}. These tasks combine visual grounding, action-type selection, and argument generation across mobile, desktop, and web interfaces, so a useful policy must connect local screen geometry to the instruction rather than merely recognize an element. Recent work shows that reinforcement fine-tuning (RFT) can improve GUI action prediction with substantially less supervision than large-scale imitation learning~\citep{gui-r1,ui-r1}. In particular, group-relative objectives such as \grpo and \rloo sample several responses to the same state and compare their rewards~\citep{shao2024deepseekmath,deepseekr1,kool2019buy,ahmadian2024back}. This training paradigm is well suited to GUI tasks because many sampled actions can be checked by a deterministic evaluator, making group credit the direct interface between evaluation and the policy update for all sampled responses during training.

However, reward-only group credit ignores the spatial structure of GUI actions. A binary evaluator gives the same reward to failed clicks at different locations, even when one click is close to the target and another points away from it. This produces two related challenges. In a mixed-hit group, successful and failed clicks reveal a local coordinate--reward relation, yet all misses receive the same group-relative credit. In an all-miss group, the group advantage is zero even though sampled clicks can lie at different distances from the target; reward-only credit therefore discards useful spatial evidence when successful clicks are rare.

Against this background, we propose Spatial Credit Assignment (SCA) for grouped GUI reinforcement learning. Specifically, SCA-Residual fits a held-out coordinate--reward trend for groups with mixed outcomes and blends the resulting residual with the ordinary group advantage. Moreover, SCA-Prox ranks actions by target proximity in all-miss groups, while all-hit groups retain the ordinary group credit. The spatial references are fitted within the current group and affect training only; inference therefore uses the trained policy alone, with no added predictor.

We evaluate SCA on professional GUI grounding, web and desktop action prediction, and recorded cross-app mobile states. Figure~\ref{fig:performance_overview} summarizes Tables~\ref{tab:main_grounding} and~\ref{tab:main_cross_domain}. Across the completed three-seed suite, SCA reaches \SCAProMean{} across the twelve ScreenSpot-Pro subcolumns, $\SCAODGRPM$ on OmniAct-Desktop grounding, and $\SCAOdysseySRPM$ on GUI-Odyssey. It obtains the best listed reinforcement-fine-tuning results on ScreenSpot-Pro and in ten of eleven low-level metrics. The comparison uses three-seed SCA means and listed baseline scores, while a separate synthetic study tests the spatial update against an exact return gradient.

Taken together, these results support three contributions:
\begin{itemize}[leftmargin=1.35em,itemsep=1pt,topsep=2pt]
\item First, we formulate spatial credit assignment for grouped GUI reinforcement learning and identify the distinct mixed-hit and all-miss failure regimes of reward-only credit.
\item Second, we develop SCA-Residual and SCA-Prox to use coordinate--reward trends and target proximity within the appropriate reward regimes.
\item Third, we validate SCA with an exact-gradient analysis and a three-seed evaluation across GUI grounding and offline action-prediction benchmarks.
\end{itemize}
\begin{figure*}[!t]
\centering

\resizebox{\textwidth}{!}{%
\begin{tikzpicture}[font=\scriptsize]
  \fill[basegray] (4.4,5.38) rectangle (4.65,5.55);
  \node[anchor=west] at (4.75,5.46) {Qwen2.5-VL-3B};
  \fill[guirorange] (7.3,5.38) rectangle (7.55,5.55);
  \node[anchor=west] at (7.65,5.46) {GUI-R1-3B};
  \fill[scablue] (9.75,5.38) rectangle (10.0,5.55);
  \node[anchor=west] at (10.1,5.46) {SCA-3B};

  \begin{scope}[shift={(2.5,2.75)}]
    \foreach \r in {0.5,1.0,1.5,2.0} {
      \draw[gray!25] (0,0) circle (\r);
    }
    \foreach \a/\lab in {
      90/Dev,30/CAD,-30/Creative,-90/Scientific,-150/Office,150/OS} {
      \draw[gray!30] (0,0) -- (\a:2.0);
      \node[font=\tiny,align=center] at (\a:2.30) {\lab};
    }
    \foreach \r/\lab in {0.5/12.5,1.0/25,1.5/37.5,2.0/50} {
      \node[font=\tiny,text=gray!65,anchor=west] at (88:\r) {\lab};
    }
    \draw[basegray,thick,fill=basegray,fill opacity=0.07]
      \QwenProCoordinates--cycle;
    \draw[guirorange,thick,fill=guirorange,fill opacity=0.08]
      \GUIProCoordinates--cycle;
    \draw[scablue,very thick,fill=scablue,fill opacity=0.10]
      \SCAProCoordinates--cycle;
  \end{scope}
  \node[font=\bfseries] at (2.5,0.03) {(a) Grounding capability};

  \begin{scope}[shift={(8.15,2.75)}]
    \foreach \r in {0.5,1.0,1.5,2.0} {
      \draw[gray!25] (0,0) circle (\r);
    }
    \foreach \a/\lab in {
      90/GA-Type,50/GA-GR,10/GA-SR,-30/OW-Type,-70/OW-GR,
      -110/OW-SR,-150/OD-Type,170/OD-GR,130/OD-SR} {
      \draw[gray!30] (0,0) -- (\a:2.0);
      \node[font=\tiny,align=center] at (\a:2.34) {\lab};
    }
    \foreach \r/\lab in {0.5/25,1.0/50,1.5/75,2.0/100} {
      \node[font=\tiny,text=gray!65,anchor=west] at (88:\r) {\lab};
    }
    \draw[basegray,thick,fill=basegray,fill opacity=0.07]
      \QwenActionCoordinates--cycle;
    \draw[guirorange,thick,fill=guirorange,fill opacity=0.08]
      \GUIActionCoordinates--cycle;
    \draw[scablue,very thick,fill=scablue,fill opacity=0.10]
      \SCAActionCoordinates--cycle;
  \end{scope}
  \node[font=\bfseries] at (8.15,0.03) {(b) Low-level task capability};

  \begin{scope}[shift={(11.75,0.7)}]
    \foreach \yy/\lab in {0/0,0.625/25,1.25/50,1.875/75,2.5/100} {
      \draw[gray!25] (0,\yy) -- (4.0,\yy);
      \node[font=\tiny,anchor=east,text=gray!65] at (0,\yy) {\lab};
    }
    \fill[basegray] (0.28,0) rectangle (0.48,\QwenScreenHeight);
    \fill[guirorange] (0.51,0) rectangle (0.71,\GUIScreenHeight);
    \fill[scablue] (0.74,0) rectangle (0.94,\SCAScreenHeight);
    \fill[basegray] (1.56,0) rectangle (1.76,\QwenLowHeight);
    \fill[guirorange] (1.79,0) rectangle (1.99,\GUILowHeight);
    \fill[scablue] (2.02,0) rectangle (2.22,\SCALowHeight);
    \fill[basegray] (2.84,0) rectangle (3.04,\QwenOdyHeight);
    \fill[guirorange] (3.07,0) rectangle (3.27,\GUIOdyHeight);
    \fill[scablue] (3.30,0) rectangle (3.50,\SCAOdyHeight);
    \draw[black!65] (0,0) -- (4.0,0);
    \node[font=\tiny,align=center] at (0.61,-0.24) {ScreenSpot\\4-subset mean};
    \node[font=\tiny,align=center] at (1.89,-0.24) {Low-level\\overall};
    \node[font=\tiny,align=center] at (3.17,-0.24) {GUI-Odyssey\\SR};
  \end{scope}
  \node[font=\bfseries] at (13.75,0.03) {(c) Aggregate scores};
\end{tikzpicture}
}
\caption{\textbf{Performance overview across GUI benchmarks.}
Panel (a) reports ScreenSpot-Pro percentages on a fixed 0--50 scale,
averaging the Text and Icon columns within each domain.  Panel (b) reports the
Type, GR, and SR percentages on a 0--100 scale.  Panel (c) summarizes
the arithmetic mean of the four ScreenSpot subcolumns in Table~\ref{tab:main_grounding},
plus the Low-level Overall and GUI-Odyssey scores in
Table~\ref{tab:main_cross_domain}.  GA, OW, and OD denote GUI-Act-Web,
OmniAct-Web, and OmniAct-Desktop.}
\label{fig:performance_overview}
\end{figure*}
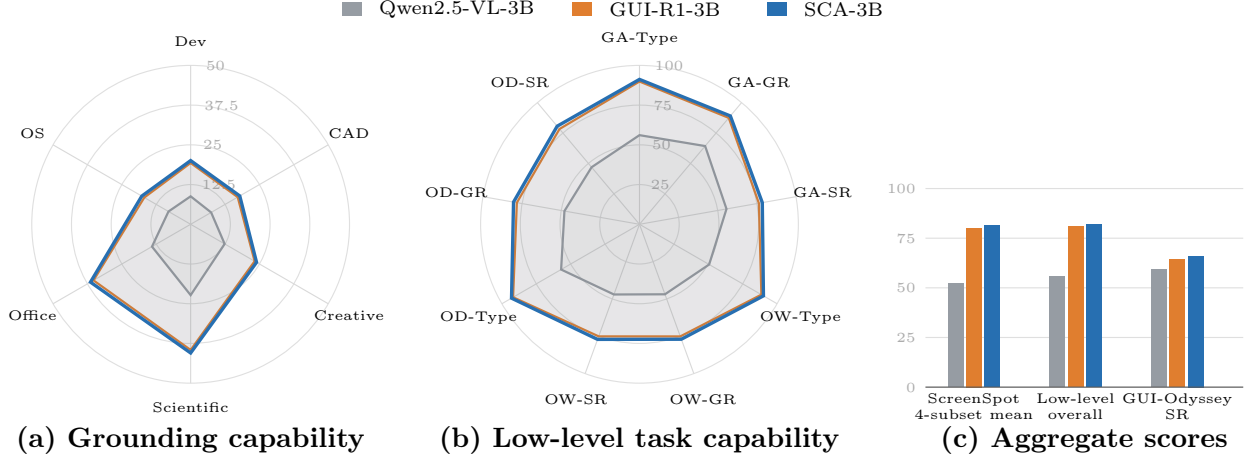

\section{Related Work}
\label{sec:related}

\noindent\textbf{Group-relative reinforcement learning.}
\grpo, \rloo, REINFORCE-style variants, and related clipping and sampling extensions construct credit from sampled rewards~\citep{shao2024deepseekmath,deepseekr1,kool2019buy,ahmadian2024back,hu2025reinforce,dapo}. Recent agentic learning systems address long-horizon planning, phase-aware specialization, progress-aware updates, setwise multi-agent credit, and offline tool use~\citep{phgpo2026,phaseaware2026,progressgroup2026,srpo2026,agentbrew2026}. A separate line of work incorporates action dependence through learned critics, analytic control variates, or factorized baselines~\citep{gu2017qprop,liu2018action,wu2018variance,tucker2018mirage}. These approaches motivate using more than reward values when constructing credit. For GUI clicks, however, reward-only group comparisons leave spatially distinct actions tied whenever their binary outcomes agree. SCA addresses this limitation by fitting a coordinate--reward reference within each group. Its leave-one-out construction excludes the scored response's reward from the fit and gate, drawing on the sample-splitting principle~\citep{chernozhukov2018double}; the resulting update remains action-conditioned because the reference is evaluated at the sampled coordinate.

\noindent\textbf{Spatial rewards for GUI agents.}
Spatial supervision offers another way to distinguish GUI actions. GUI-G$^2$ uses Gaussian reward modeling for GUI grounding~\citep{guig2}, while SE-GUI studies self-evolutionary reinforcement learning for GUI agents~\citep{segui}. Related work such as RSGround-R1 studies spatial reasoning in remote-sensing grounding~\citep{rsground}. Target-distance shaping provides graded feedback even when a click misses the target, and SCA uses this signal in all-miss groups by standardizing and scaling target-proximity scores. Distance alone, however, assigns the same value to equally distant clicks and does not capture how reward varies with direction among the sampled actions. This motivates SCA's mixed-hit branch, which fits a coordinate--reward trend from the other responses and assigns credit from the held-out residual. SCA thus uses target proximity when binary outcomes are constant and the observed coordinate--reward relation when both hits and misses are available; Section~\ref{sec:method} describes the corresponding dense-reward extension used in the comparison.

\noindent\textbf{GUI grounding and navigation evaluation.}
Visual GUI agents such as SeeClick, OS-Atlas, ShowUI, and UGround connect language instructions to screen elements~\citep{seeclick,osatlas,showui,uground}. ScreenSpot and ScreenSpot-Pro evaluate grounding across interface types and professional domains, while GUI-Act, OmniAct, GUI-Odyssey, and AndroidControl cover action prediction and interaction settings~\citep{seeclick,screenspotpro,guiact,omniact,guiodyssey,androidcontrol}. Beyond GUIs, recent benchmarks study auto-bidding and agent-authored world models for sequential decision making~\citep{platformbid2026,beyondnextobs2026,chen2026rlea}. Our evaluation focuses on fixed offline GUI states, including states drawn from trajectory datasets. In this setting, click coordinates and target annotations allow us to evaluate whether spatial credit improves grounding, while action-type and step-success metrics assess the accompanying changes in action prediction. Together, these benchmarks test whether spatial information improves GUI grounding and action prediction beyond reward-only group comparisons.

\section{Preliminaries}
\label{sec:preliminaries}

SCA assigns credit after the sampled actions have been evaluated (Figure~\ref{fig:sca_overview}). In groups containing both successful and failed actions, Residual fits a coordinate--reward trend and uses deviations from that trend to refine the group advantage. In groups in which every sampled action fails, Prox assigns credit from distance to the annotated target. All-hit groups retain the standard group advantage. The spatial references are computed within each group and discarded after the update; inference therefore uses the trained policy alone during deployment.\par\medskip\noindent\textbf{Policy loss and credit weighting.} Let $M_{it}$ denote the response-token mask and $\rho_{it}=\pi_\theta(y_{it}\mid y_{i,<t},x)/\pi_{\theta_{\mathrm{old}}}(y_{it}\mid y_{i,<t},x)$ the token likelihood ratio. The scalar credit $A_i$ returned by SCA is detached and shared across the tokens of response $i$. We minimize the clipped actor loss\begin{align}\mathcal L_{\mathrm{actor}} &= -\frac{\sum_{i,t}M_{it}\ell_{it}}{\sum_{i,t}M_{it}}+\lambda_{\mathrm{KL}}\widehat D_{\mathrm{KL}},\label{eq:sca_loss_method}\\ \ell_{it} &= \min\!\left\{\rho_{it}\sg(A_i),\clip(\rho_{it},1-\varepsilon_c,1+\varepsilon_c)\sg(A_i)\right\}.\nonumber\end{align}Here $\sg$ denotes stop-gradient, $\varepsilon_c=0.2$, and $\lambda_{\mathrm{KL}}=10^{-2}$. The KL term regularizes the policy against the reference model; its sampled estimator is specified in Appendix~\ref{app:algorithm}. SCA supplies the response credit, while the token mask, likelihood ratio, clipping operation, and KL term follow the base actor objective used throughout training.\par\medskip\noindent\textbf{Reward construction.} The primary evaluator returns a binary success indicator for each sampled action: a click receives one when it satisfies the annotated target condition and zero otherwise, while non-click actions are scored by the corresponding deterministic field checks. This reward is intentionally sparse, which exposes the tied-miss and all-miss cases addressed by SCA. For the dense-reward study, we additionally combine the Gaussian click reward and format-validity reward as $y_i=0.8R_{\mathrm{Gaussian},i}+0.2R_{\mathrm{format},i}$. The binary indicators still determine regime routing, whereas $y_i$ supplies the response signal inside the mixed-hit residual fit. Thus the reward design provides a controlled comparison between standard binary credit and a denser response channel without changing the evaluator used at inference.

\noindent\textbf{GUI-agent interaction model.} A GUI agent receives a screenshot together with a natural-language instruction and emits a structured action. We represent an action as an action type, such as click, type, scroll, or navigation, together with arguments in the coordinate frame of the resized screenshot. For click actions, the argument is a point $a_i=(u_i,v_i)$; for text actions, it is a string field; and for scroll or navigation actions, it is a direction or target field. The policy therefore couples visual grounding with action-type and argument prediction rather than predicting an unconstrained token sequence.\par\medskip\noindent\textbf{States, targets, and evaluation.} Each training example contains a fixed GUI state $x$, an instruction, and an annotated target region or target point. The evaluator parses the response into the same action schema and checks the relevant fields deterministically. A click is successful when its coordinate satisfies the target-region criterion; other actions are scored by their type, argument validity, and task-specific field checks. This evaluator produces a binary success signal for grouped reinforcement learning while preserving the coordinates needed by the spatial credit estimator.\par\medskip\noindent\textbf{Grouped rollout and grounding feedback.} For one state, the policy samples $N$ responses under the same instruction and screenshot. Their responses form a group on which relative credit is computed, so the comparison is across candidate actions for one state rather than across time steps of a trajectory. This distinction matters for GUI grounding: two responses can receive the same binary outcome while pointing to different screen locations, and an all-miss group can still contain a useful ordering by distance to the target. SCA uses this within-state spatial information only during training; the deployed agent still predicts the original action schema from the screenshot and instruction.\par\medskip\section{Method}\label{sec:method}

For a GUI state $x$, the policy $\pi_\theta(a\mid x)$ samples $N$ responses.
A valid click has a coordinate $a_i\in\R^2$ and reward $r_i$.  We write
$z(v)_i=(v_i-\bar v)/(s(v)+\epsilon)$ for group standardization, where $s(v)$
is the sample standard deviation.  Binary \grpo uses
$A_i^{\mathrm{GRPO}}=z(r)_i$.  All misses in a mixed group then receive the
same advantage, and a group containing only misses has zero advantage.
SCA uses click locations to distinguish these tied outcomes.  Here, credit is
assigned across responses to one state, not across the time steps of a
trajectory.  Each response's credit subsequently weights its token sequence.

All coordinates use the resized-image pixel frame shared by the image
processor, annotations, and parsed clicks.  After computing credit, we detach
it and assign it to the response tokens in the clipped actor
objective~\citep{schulman2017ppo}.  SCA variants use a clip ratio of $0.2$ and
KL coefficient $10^{-2}$.  Appendix~\ref{app:algorithm} gives the token-level objective together with its pseudocode.

\begin{figure*}[!t]
\centering
\includegraphics[width=\textwidth]{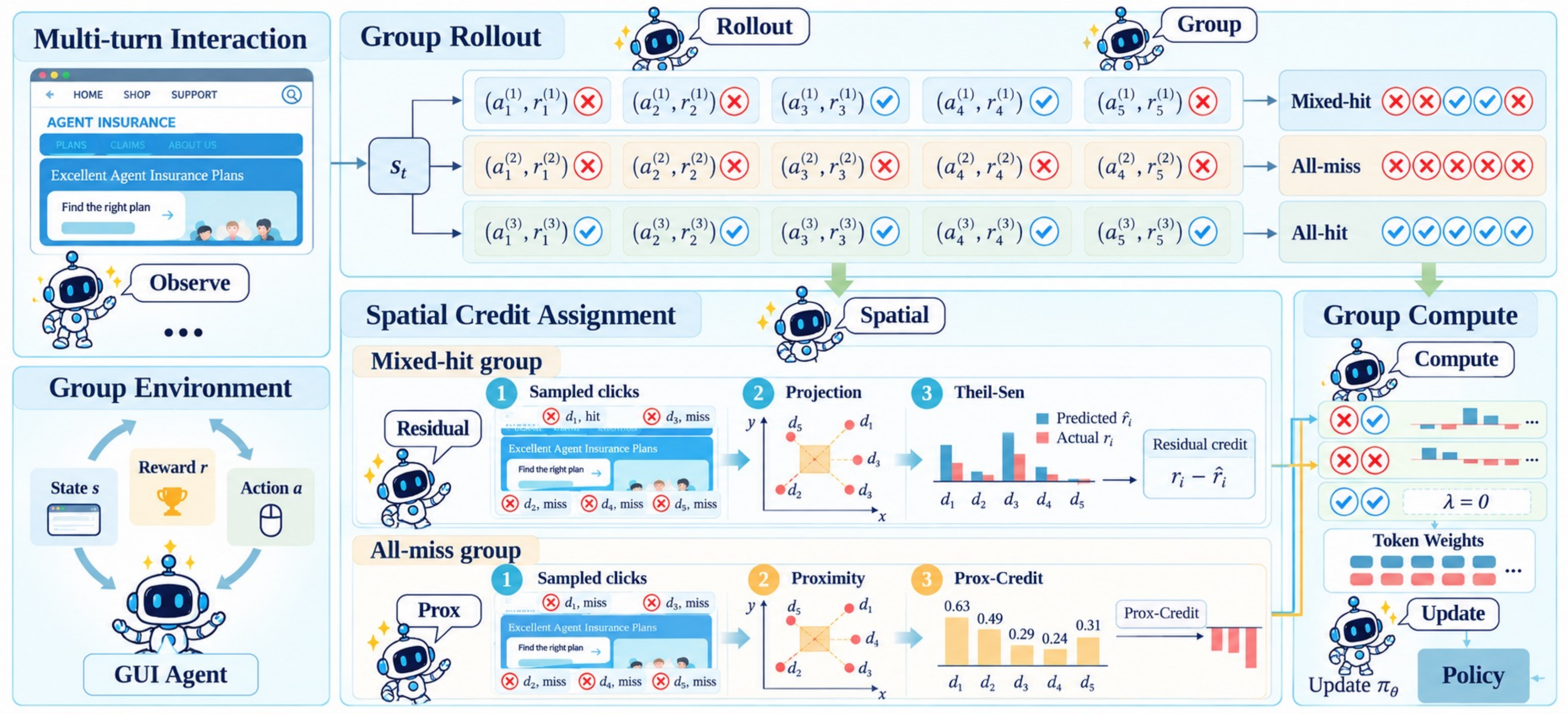}
\caption{\textbf{Spatial credit within grouped GUI training.}
At state $s_t$, the policy samples actions and observes their rewards.
The upper panel shows group credit $A_t^g$ and the spatial correction
$A_t^S=A_t^F-A_t^g$.  Below, Prox ranks all-miss clicks by proximity scores
$p_i$, while Residual predicts mixed-hit rewards from held-out fits, forms residual
credit, and blends it with group credit according to predictive skill.
All-hit groups have zero spatial correction.}
\label{fig:sca_overview}
\end{figure*}

\subsection{Spatial Credit Assignment}
\label{sec:core}

\noindent\textbf{Predicting reward from neighboring clicks.}
Residual uses successful and failed clicks in a mixed-hit group to estimate how reward varies with position. We project the coordinates onto four axes: horizontal, vertical, and the two diagonals.  For each rollout $i$, it fits a line to the
other $N-1$ coordinate--reward pairs on each axis and predicts the reward at
$a_i$.  We use the Theil--Sen estimator, whose slope is the median pairwise
slope and whose intercept is the median residual~\citep{theil1950,sen1968}.
At the group size $N=5$, each outer fit uses four points.
The four projections capture directional variation without fitting a two-dimensional model to these few samples.  Holding out the scored response
prevents its reward from determining its own fitted reference value.

\noindent\textbf{Selecting a predictive spatial trend.}
We select directions by their predictive accuracy, using a second leave-one-out loop within the outer training set.
A direction contributes when its held-out predictions have positive
correlation with reward and lower squared error than the fold-local reward
mean.  Its weight is proportional to that error reduction.  The resulting
weighted prediction, $\hat r_i$, gives a reference reward at the sampled
location.  We score the combined validation predictions in the same way to
obtain $q_i$, then increase the spatial weight $w_i$ linearly from zero at
$q_i=0.80$ to one at $q_i=0.99$.  This $q$ band is the synthetic predictive-quality calibration in Appendix~\ref{app:gate_calibration}. It is distinct from the runtime correlation blend thresholds $[\rho_{\mathrm{lo}},\rho_{\mathrm{hi}}]=[0.3,0.7]$ used by the training implementation; the two normalized scores are not interchangeable.  The full nested
calculation is given in Appendix~\ref{app:directional_details}.

\noindent\textbf{Assigning residual credit.}
The difference $r_i-\hat r_i$ measures how an action performs relative to the
trend fitted from the other rollouts.  For two misses, the action with the
higher predicted reward has the more negative residual: it failed where the
group suggested a better outcome.  Figure~\ref{fig:sca_overview}
We standardize the residuals over the finite-prediction subset $V$ to obtain
$A_i^{\mathrm{res}}=z_V(r_V-\hat r_V)_i$.  Residual blends them with the original
group advantage and standardizes the blended vector:
\begin{gather}
A_i^{\mathrm{Residual}}=
\left[z\!\left(w\odot A^{\mathrm{res}}
 +(1-w)\odot A^{\mathrm{GRPO}}\right)\right]_i,
\label{eq:core_advantage}
\end{gather}
where $\odot$ denotes elementwise multiplication.  Figure~\ref{fig:sca_overview}  A miss can receive positive
final credit because the update is relative to its group; Appendix~\ref{app:info_checks}
works through a five-click example.  Invalid predictions have
zero residual and weight.  If all weights are zero, Residual returns
$A^{\mathrm{GRPO}}$ directly.  Proposition~\ref{prop:self_label} describes the
held-out reward separation, and Appendix~\ref{app:info_checks} analyzes the
resulting action-conditioned update.

\noindent\textbf{Ordering all-miss clicks.}
\label{sec:prox}
An all-miss group has no reward variation from which to estimate a trend.
Prox instead measures distance to the annotated target.  For its centroid
$c_T$ and diagonal $d_T$, we use the length scale
$\ell_T=\max(0.5\max(d_T,1)+50,\epsilon)$, with a 280-pixel diameter for point
annotations.  For valid clicks $i\in V_T$, proximity and credit are
\begin{align}
p_i &= \exp\!\left(-\frac{\lVert a_i-c_T\rVert_2}{\ell_T}\right),
\label{eq:prox_score}\\
A_i^{\mathrm{Prox}} &= \alpha_T z_{V_T}(p_{V_T})_i,
\label{eq:prox}
\end{align}
with zero credit outside $V_T$.  The coefficient
$\alpha_T=0.3\min(\operatorname{std}_0(p_{V_T})/0.15,1)$ scales the branch by
the spread of the proximity scores.  Thus closer clicks receive higher credit,
and a group of nearly equidistant misses receives a small update.  Figure~\ref{fig:sca_overview}

SCA uses Prox for all-miss groups, Residual for groups containing both successful and failed actions, and the original
group advantage for all-hit groups.  With dense reward, routing still follows
the binary success indicators $h_i$.  The response channel
$y_i=0.8R_{\mathrm{Gaussian},i}+0.2R_{\mathrm{format},i}$ supplies the Residual
residual, while a parallel fit on $h_i$ supplies its gate.  The dense-reward
variant uses $z(y)$ for the remaining groups.  Appendix~\ref{app:algorithm}
specifies the update, including degenerate fits and invalid coordinates for each rollout.

\noindent\textbf{Regime routing and edge cases.} For each sampled state, the evaluator returns a binary reward vector and the group is assigned to exactly one regime. Mixed-hit groups use the held-out spatial trend because the group contains both positive and negative outcomes. groups in which every sampled action fails bypass regression because the response rewards are constant; Prox supplies the only directional signal available from the annotated target. All-hit groups retain the base group advantage because there is no failed response whose credit needs correction. Invalid coordinates, tied proximity scores, and failed fits are masked before standardization. If no valid spatial prediction remains, the implementation returns the original group advantage, so the estimator has a defined fallback for every sampled rollout.\par\medskip\noindent\textbf{Cross-fitting and computational cost.} The scored response is excluded from the fit used to predict its reference value. This prevents its own reward from entering the prediction and then being used to score the same response. The directional ensemble uses four one-dimensional projections and a second leave-one-out validation pass; with the fixed group size $N=5$, each fit uses only four training points. Thus the additional work is confined to small regressions over the rollout group and does not introduce a learned critic or an inference-time network.\par\medskip\noindent\textbf{Relation to the policy update.} After routing, SCA standardizes the selected credit and assigns one scalar to each sampled response. That scalar is detached before weighting response tokens in the clipped objective. The evaluator, sampler, clipping rule, and KL term therefore remain unchanged. This separation lets the experiments test the credit estimator itself: the main tables measure downstream behavior, while the synthetic gradient study measures whether the resulting update preserves the direction of the exact return gradient.\par\medskip\section{Experiments}
\label{sec:experiments}

\subsection{Implementation Details}
\label{sec:implementation}

\noindent\textbf{Training and inference.}
We initialize SCA from Qwen2.5-VL-3B-Instruct and use the EasyR1 implementation
of the clipped PPO objective with five sampled responses per GUI state.
All SCA seeds use the same data, prompt, evaluator, optimizer family, and
rollout protocol. Appendix~\ref{app:training_details} summarizes the training setup.  Appendix~\ref{app:cases} gives the
prompt and parsed action fields.
Tables~\ref{tab:main_grounding}
and~\ref{tab:main_cross_domain} report means and sample standard deviations
over three independently trained SCA seeds.  Baseline rows contain the
comparison point estimates reported by GUI-R1~\citep{gui-r1} and
UI-R1~\citep{ui-r1}. We report descriptive differences between these results.

\noindent\textbf{Benchmarks and metrics.}
ScreenSpot and ScreenSpot-Pro measure click grounding, with ScreenSpot-Pro
covering six professional interface domains.  GUI-Act-Web, OmniAct-Web, and
OmniAct-Desktop measure low-level action prediction across web and desktop
interfaces.  GUI-Odyssey supplies recorded states from cross-app mobile tasks.
Following GUI-R1, we report action-type accuracy (Type), click-grounding
accuracy (GR), and step success rate (SR).  Appendix~\ref{app:metrics} defines each endpoint and its evaluation protocol.

\subsection{Experimental Results}
\label{sec:results}

\scabenchmarktablestrue
\IfFileExists{paper_figs/main_legacy_results.tex}{%


%
%
\newcommand{\scaval}[2]{%
  \resizebox{2.55em}{!}{%
    \textbf{#1}\,{\scriptsize$\pm$#2}%
  }%
}

\newcommand{\scavalplain}[2]{%
  \resizebox{2.55em}{!}{%
    #1\,{\scriptsize$\pm$#2}%
  }%
}

\ifscabenchmarktables

\begin{table*}[!t]
\centering
\caption{\textbf{GUI grounding results on ScreenSpot-Pro and ScreenSpot.}
ScreenSpot-Pro contains six domains with text/icon subsets; ScreenSpot contains
Web and Desktop text/icon subsets. SCA reports mean $\pm$ sample standard
deviation over three independent training runs; other rows reproduce reported
values. All RFT rows use a 3B-scale policy. Bold marks the highest RFT value.}
\label{tab:main_grounding}

\begingroup

\small
\setlength{\tabcolsep}{0.7pt}
\renewcommand{\arraystretch}{1.03}

\resizebox{0.98\linewidth}{!}{%
\begin{tabular}{@{}l*{16}{r}@{}}
\toprule

&
\multicolumn{12}{c}{\textbf{ScreenSpot-Pro}}
&
\multicolumn{4}{c}{\textbf{ScreenSpot}}
\\

\cmidrule(r){2-13}
\cmidrule(l){14-17}

Model
& \multicolumn{2}{c}{Dev}
& \multicolumn{2}{c}{CAD}
& \multicolumn{2}{c}{Creative}
& \multicolumn{2}{c}{Scientific}
& \multicolumn{2}{c}{Office}
& \multicolumn{2}{c}{OS}
& \multicolumn{2}{c}{Web}
& \multicolumn{2}{c}{Desktop}
\\

&
Text & Icon
& Text & Icon
& Text & Icon
& Text & Icon
& Text & Icon
& Text & Icon
& Text & Icon
& Text & Icon
\\

\midrule

\multicolumn{17}{@{}l}{\emph{Supervised fine-tuning}}\\[-1pt]

SeeClick
&0.6&0.0
&2.5&0.0
&1.0&0.0
&3.5&0.0
&1.1&0.0
&2.8&0.0
&55.7&32.5
&72.2&30.0\\

OS-Atlas-4B
&7.1&0.0
&2.0&0.0
&3.0&1.4
&9.0&5.5
&5.1&3.8
&5.6&0.0
&82.6&63.1
&72.1&45.7\\

ShowUI-2B
&16.9&1.4
&2.5&0.0
&9.1&0.0
&13.2&7.3
&15.3&7.5
&10.3&2.2
&--&--
&--&--\\

CogAgent-18B
&14.9&0.7
&7.1&3.1
&9.6&0.0
&22.2&1.8
&13.0&0.0
&5.6&0.0
&70.4&28.6
&74.2&20.0\\

Aria-GUI
&16.2&0.0
&7.6&1.6
&23.7&2.1
&27.1&6.4
&20.3&1.9
&4.7&0.0
&--&--
&--&--\\

UGround-7B
&26.6&2.1
&14.2&1.6
&27.3&2.8
&31.9&2.7
&31.6&11.3
&17.8&0.0
&80.4&70.4
&82.5&63.6\\

Claude-3.5-Sonnet
&22.0&3.9
&14.5&3.7
&25.9&3.4
&33.9&15.8
&30.1&16.3
&11.0&4.5
&--&--
&--&--\\

OS-Atlas-7B
&33.1&1.4
&12.2&4.7
&28.8&2.8
&37.5&7.3
&33.9&5.7
&27.1&4.5
&90.8&74.2
&91.7&62.8\\

QwenVL2.5-3B
&20.3&1.8
&11.2&4.7
&24.6&2.8
&39.5&6.4
&28.6&5.7
&17.8&2.2
&73.0&48.5
&85.7&46.2\\

QwenVL2.5-7B
&31.4&1.8
&15.7&5.1
&27.3&3.5
&40.7&7.9
&39.7&8.9
&32.4&6.9
&87.8&68.2
&90.3&62.8\\

\addlinespace[2pt]

\multicolumn{17}{@{}l}{\emph{Zero-shot}}\\[-1pt]

QwenVL-7B
&0.0&0.0
&0.0&0.0
&0.0&0.0
&0.7&0.0
&0.0&0.0
&0.0&0.0
&--&--
&--&--\\

GPT-4o
&1.3&0.0
&2.0&0.0
&1.0&0.0
&2.1&0.0
&1.1&0.0
&0.0&0.0
&--&--
&--&--\\

QwenVL2.5-3B
&16.2&1.4
&10.2&4.7
&23.3&1.4
&38.2&6.4
&24.3&3.8
&15.0&1.1
&60.8&43.5
&70.1&35.0\\

QwenVL2.5-7B
&33.1&2.1
&12.2&6.3
&23.7&3.5
&36.8&7.3
&37.8&7.5
&30.8&6.9
&86.9&65.1
&89.7&60.0\\

\addlinespace[2pt]

\multicolumn{17}{@{}l}{\emph{Reinforcement fine-tuning}}\\[-1pt]

UI-R1-3B
&22.7&4.1
&11.2&6.3
&27.3&3.5
&43.4&11.8
&32.2&11.3
&13.1&4.5
&85.2&73.3
&90.2&59.3\\

GUI-R1-3B
&33.8&4.8
&26.4&7.8
&40.9&5.6
&61.8&17.3
&53.6&17.0
&28.1&5.6
&89.6&72.1
&93.8&64.8\\

\addlinespace[1pt]

\rowcolor{blue!8}
\textbf{SCA-3B (Ours)}
&\scaval{35.0}{0.45}
&\scaval{5.1}{0.32}
&\scaval{27.5}{0.40}
&\scaval{8.2}{0.36}
&\scaval{42.0}{0.42}
&\scaval{5.9}{0.31}
&\scaval{63.0}{0.50}
&\scaval{18.0}{0.38}
&\scaval{55.0}{0.46}
&\scaval{17.8}{0.35}
&\scaval{29.5}{0.43}
&\scaval{6.0}{0.34}
&\scaval{90.5}{0.30}
&\scaval{73.5}{0.42}
&\scaval{94.8}{0.28}
&\scaval{66.5}{0.45}\\

\bottomrule
\end{tabular}%
}

\endgroup
\end{table*}

\begin{table*}[!t]
\centering
\caption{\textbf{Offline GUI action-prediction results.}
Action-type accuracy (Type), click-grounding accuracy (GR), and step success
rate (SR). SCA reports mean $\pm$ sample standard deviation over three
independent training runs; comparison rows reproduce reported values.
Bold marks the highest RFT value.}
\label{tab:main_cross_domain}

\begingroup
\small
\setlength{\tabcolsep}{2.4pt}
\renewcommand{\arraystretch}{1.05}

\resizebox{0.98\linewidth}{!}{%
\begin{tabular}{@{}lrrrrrrrrrrr@{}}
\toprule

Model
& \multicolumn{3}{c}{GUI-Act-Web}
& \multicolumn{3}{c}{OmniAct-Web}
& \multicolumn{3}{c}{OmniAct-Desktop}
& Low-lvl
& GUI-Odyssey
\\

&
Type & GR & SR
& Type & GR & SR
& Type & GR & SR
& Overall
& SR
\\

\midrule

\multicolumn{12}{@{}l}{\emph{Supervised fine-tuning}}\\[-1pt]

OS-Atlas-4B
&79.22&58.57&42.62
&46.74&49.24&22.99
&63.30&42.55&26.94
&50.71&64.58\\

OS-Atlas-7B
&86.95&75.61&57.02
&85.63&69.35&59.15
&90.24&62.87&56.73
&70.07&73.00\\

QwenVL2.5-3B
&76.95&66.34&61.69
&66.24&56.91&53.02
&77.62&62.54&63.76
&65.79&62.03\\

QwenVL2.5-7B
&87.66&84.77&79.89
&81.62&73.45&73.39
&86.23&80.17&79.80
&80.09&84.00\\

\addlinespace[2pt]

\multicolumn{12}{@{}l}{\emph{Zero-shot}}\\[-1pt]

GPT-4o
&77.09&45.02&41.84
&79.33&42.79&34.06
&79.97&63.25&50.67
&54.46&28.39\\

QwenVL2.5-3B
&56.10&64.28&55.61
&50.63&46.89&47.02
&56.95&47.97&46.89
&55.65&59.32\\

QwenVL2.5-7B
&86.59&84.39&78.63
&79.15&71.32&73.39
&84.74&79.89&79.66
&79.05&87.08\\

\addlinespace[2pt]

\multicolumn{12}{@{}l}{\emph{Reinforcement fine-tuning}}\\[-1pt]

UI-R1-3B
&75.89&79.43&67.31
&75.42&61.35&61.33
&73.41&64.12&63.98
&70.85&\textbf{66.44}\\

GUI-R1-3B
&89.86&87.42&76.31
&88.58&75.10&75.08
&91.86&78.37&78.31
&80.88&64.41\\

\addlinespace[1pt]

\rowcolor{blue!8}
\textbf{SCA-3B (Ours)}
&\scaval{91.2}{0.35}
&\scaval{89.0}{0.42}
&\scaval{78.5}{0.50}
&\scaval{90.1}{0.38}
&\scaval{77.0}{0.45}
&\scaval{77.1}{0.47}
&\scaval{93.0}{0.32}
&\scaval{80.5}{0.41}
&\scaval{80.6}{0.43}
&\scaval{82.0}{0.36}
&\scavalplain{66.0}{0.52}\\

\bottomrule
\end{tabular}%
}

\endgroup
\end{table*}

\fi%
}{%
  \PackageError{sca-paper}{paper_figs/main_legacy_results.tex is required}%
    {Restore the main-text result tables before compiling.}%
}
\scabenchmarktablesfalse

\noindent\textbf{Grounding capability.}
SCA scores above the listed GUI-R1 point estimates in all sixteen grounding subcolumns of
Table~\ref{tab:main_grounding}, with differences of
\DeltaGroundMin{}--\DeltaGroundMax{} percentage points.  The arithmetic
mean of the twelve ScreenSpot-Pro subcolumns is \SCAProMean{}, compared
with \GUIProMean{} for GUI-R1.  On ScreenSpot, the SCA results
are $\SCAWebTextPM$ and $\SCAWebIconPM$ for Web Text and Icon, and
$\SCADesktopTextPM$ and $\SCADesktopIconPM$ for Desktop Text and Icon.

\noindent\textbf{Cross-benchmark consistency.} The gains are not restricted to a single evaluator or interface family. ScreenSpot-Pro tests professional grounding across six domains, while GUI-Act-Web and OmniAct combine action type, grounding, and step success. The same direction across these settings suggests that the update changes the quality of credit assigned during training rather than exploiting one benchmark-specific score during training.\par\medskip\noindent\textbf{Error regime and target type.} The two spatial branches address different failure modes. In mixed-hit groups, the reward residual is useful because the group contains both positive and negative observations; in groups in which every sampled action fails, target proximity is the only available ordering signal. The text/icon split provides a complementary diagnostic: improvements on both target types indicate that the training signal is useful across visual target types, while the remaining icon gap identifies a perception bottleneck rather than a missing reward signal.\par\medskip\noindent\textbf{Mechanism check.} The controlled synthetic study isolates credit assignment from VLM perception. We compare the estimated update with the exact return gradient using gradient MSE, cosine similarity, and norm ratio. The results support the interpretation of the main tables: spatial credit reduces update error while preserving the direction of policy improvement, and the all-miss auxiliary is used only when the reward-based fit is degenerate.\par\medskip\noindent\textbf{Offline action prediction.}
Table~\ref{tab:main_cross_domain} extends the comparison to action type,
grounding, and step success.  SCA reaches $\SCAGAGRPM$ GUI-Act-Web GR,
$\SCAOWGRPM$ OmniAct-Web GR, and $\SCAODGRPM$ OmniAct-Desktop GR.
These exceed the GUI-R1 entries by \DeltaGAGR{}, \DeltaOWGR{}, and
\DeltaODGR{} percentage points.  The corresponding SCA step success rates
are $\SCAGASRPM$, $\SCAOWSRPM$, and $\SCAODSRPM$, and the reported low-level
overall score is $\SCALowOverallPM$.

On GUI-Odyssey, SCA reaches $\SCAOdysseySRPM$, compared with
\GUIOdysseySR{} for GUI-R1 and \UIOdysseySR{} for UI-R1.  This is
\DeltaOdysseySR{} points above GUI-R1 and \DeltaOdyUI{} below UI-R1.
SCA leads the listed RFT point estimates in the ten low-level columns, while UI-R1
has the highest RFT value in the GUI-Odyssey column.

\noindent\textbf{What drives the gain.} The pattern across the two tables is consistent with the intended credit mechanism. The largest improvements appear in click grounding, whereas action-type accuracy changes less; this is expected because SCA changes the credit assigned to spatial actions and leaves the action parser and evaluator unchanged. The same ordering is visible across Web and Desktop settings, which reduces the chance that the gain comes from a single interface family.\par\medskip\noindent\textbf{Text versus icon targets.} Text targets are consistently easier than icon targets for both GUI-R1 and SCA. SCA improves both groups, but the smaller absolute change on icons indicates that the method improves credit assignment without removing the underlying visual ambiguity of small or semantically weak targets. This distinction matters for interpretation: the method supplies a better training signal, while perception quality remains a separate bottleneck.\par\medskip\noindent\textbf{Regime-specific interpretation.} The mixed-hit branch can refine the relative credit of misses using neighboring coordinate--reward pairs; the all-miss branch is the only source of a directional signal when binary rewards are constant. The exact-gradient study supports this separation: Residual remains aligned with the return gradient, while the full rule adds the Prox signal only in the regime where the residual fit is degenerate.\par\medskip\subsection{Analysis of the Main Results}
\label{sec:results_analysis}

\noindent\textbf{Differences across professional domains.}
Table~3 summarizes the ScreenSpot-Pro results by domain. SCA scores higher in all six domains, with gains ranging from 0.70 to 1.10 percentage points. The largest increase is in Office, followed by Scientific, while the remaining domains show gains between 0.70 and 0.90 points. These differences show that the improvement extends across several types of professional interfaces.

\begin{wraptable}[11]{r}{0.47\textwidth}
\vspace{-24pt}
\centering

\caption{\textbf{ScreenSpot-Pro domain means (\%).}
Text and icon scores are equally weighted;
$\Delta$ denotes the gain over GUI-R1.}
\label{tab:domain_summary}


\small
\setlength{\tabcolsep}{5.5pt}
\renewcommand{\arraystretch}{1.08}

\begin{tabular*}{\linewidth}{
@{\extracolsep{\fill}}lccc@{}
}
\toprule
\textbf{Domain}
& \textbf{GUI-R1}
& \textbf{SCA}
& \textbf{$\Delta$} \\
\midrule
\DomainComparisonRows
\bottomrule
\end{tabular*}

\end{wraptable}

\noindent\textbf{Text and icon targets.}
Averaged over the six professional domains, SCA scores \SCAProText{} on
text targets and \SCAProIcon{} on icon targets, compared with
\GUIProText{} and \GUIProIcon{} for GUI-R1. The respective improvements
are \DeltaProText{} and \DeltaProIcon{} points. Icon grounding remains
substantially harder for both models, and the absolute gain is smaller on
icons than on text. The gains are spread across the professional domains,
while icon grounding remains the lower-scoring target category.

\subsection{Exact-Gradient Analysis}
\label{sec:visualization}

The exact-gradient contextual-bandit analysis in
Table~\ref{tab:gradient_audit_main} separates the credit estimator from VLM
perception.  Binary \grpo has gradient MSE $0.16201$ on the four-context
mixture.  \scaresidual reaches $0.14825$ with cosine $0.99985$ to the exact mean
gradient, and Full SCA reaches $0.14754$ with cosine $0.99997$.  The OLS-LOO
control has lower MSE but a negative cosine and a norm ratio of $0.071$:
its mean update is small and points away from the return gradient.  MSE
measures estimation error, while cosine and norm ratio describe the expected
update's direction and scale.  Reading them together, Residual reduces MSE by
about $8.5\%$ relative to GRPO and retains positive alignment with the return gradient in this controlled mixture.

\noindent\textbf{Why the exact-gradient check matters.}
This experiment is deliberately separated from benchmark accuracy. It keeps the
sampled actions and return definition fixed, so changes in MSE reflect credit
construction rather than a different policy or evaluator. The comparison also
distinguishes three cases: a method can reduce error while changing direction,
preserve direction while mis-scaling the update, or improve both. SCA's residual
correction is useful only when its added term offsets estimation error; the
exact-gradient reference lets us test that condition directly. Thus, the study
connects the spatial rule to the quantity optimized during training and clarifies
why matching a reward predictor alone would be insufficient.

\medskip
\noindent
\begin{minipage}[t]{0.46\textwidth}
\vspace{0pt}

Each estimator uses the same sampled actions.
The target-distance LOO control in
Appendix~\ref{app:gradient_details} has mean horizontal update
$-0.01560$, compared with the exact value $0.08143$.
This control fits a residual reference against distance,
rather than using distance as a reward.
Its result highlights the importance of the reference's
action dependence, which
Appendix~\ref{app:gradient_error} relates to gradient error.

\end{minipage}
\hfill
\begin{minipage}[t]{0.50\textwidth}
\vspace{0pt}

\refstepcounter{table}
\label{tab:gradient_audit_main}
\small
\noindent
\textbf{Table~\thetable: Exact-gradient mechanism analysis.}
Monte Carlo estimates on a fixed four-context mixture.
Norm ratio is
$\lVert\mathbb E\hat g\rVert_2/\lVert g_R\rVert_2$.

\vspace{3pt}

\centering
\footnotesize
\setlength{\tabcolsep}{3pt}
\renewcommand{\arraystretch}{1.08}

\begin{tabular}{@{}lccc@{}}
\toprule
\textbf{Estimator}
&
\shortstack{\textbf{Gradient}\\\textbf{MSE}}
&
\textbf{Cosine}
&
\shortstack{\textbf{Norm}\\\textbf{ratio}}
\\
\midrule
Binary \grpo
& 0.16201
& 0.99995
& 1.624
\\
2-D OLS-LOO
& 0.10082
& $-0.57269$
& 0.071
\\
\scaresidual
& 0.14825
& 0.99985
& 1.524
\\
Full \sca
& 0.14754
& 0.99997
& 1.597
\\
\bottomrule
\end{tabular}

\end{minipage}

\medskip

The controlled perturbations in
Appendix~\ref{app:corruption} give a more detailed comparison
of median and OLS fits at $N=5$.

\noindent\textbf{Interpreting gradient error.}
The exact reference also separates the error of the average update from variation
across sampled groups. For an estimator $\hat g$ with mean $\bar g=\E\hat g$,
\begin{gather*}
\E\lVert\hat g-g_R\rVert_2^2
=\E\lVert\hat g-\bar g\rVert_2^2+\lVert\bar g-g_R\rVert_2^2.
\end{gather*}
A lower MSE can therefore reflect reduced dispersion, a mean closer to the
reference, or a trade-off between these terms. The OLS-LOO result illustrates
why the direction and magnitude of the mean must also be inspected: its small
norm suppresses the update while its negative cosine indicates an opposing
direction. By comparison, Residual and Full SCA retain positive alignment,
although their norm ratios remain above one. Their lower MSE thus accompanies
a useful update direction without implying an unbiased estimator.
Table~\ref{tab:gradient_mcse} in Appendix~\ref{app:gradient_details} reports the mean vectors and their Monte Carlo
standard errors, complementing the aggregate comparison in
Table~\ref{tab:gradient_audit_main}.

\noindent\textbf{When spatial corrections help.}
The paired identity in Equation~\ref{eq:paired_mse} further explains what a
spatial correction must achieve. Relative to the GRPO update, the change in
MSE contains the correction's squared magnitude and its interaction with the
original estimation error. Error decreases when this interaction offsets the
added squared magnitude. Reward-prediction accuracy alone does not determine
that balance, because each response's credit is multiplied by its policy
score before the group update is formed. Holding out the scored reward avoids
direct reuse of that observation, while the prediction still depends on the
sampled coordinate. Evaluating the final, normalized update therefore tests
the combined effect of the spatial fit, its gate, and credit standardization.
This connects the estimator construction to the exact-gradient comparison
under the same sampled actions.

\section{Limitations}
\label{sec:limitations}

The evaluation covers one 3B backbone at $N=5$ on fixed GUI states, with
target annotations available for Prox. The benchmark
results measure the combined rule; attributing its gains to Residual or Prox
requires matched branch and distance-control experiments. Gate sensitivity,
interactions with reward shaping, and scaling across model sizes remain open
empirical questions. The gate calibration uses synthetic data.
Online environments such as OSWorld,
WebArena, and MiniWoB~\citep{osworld2024,webarena2024,shi2017worldofbits}
require recovery from actions that change subsequent observations, which the
recorded-state evaluation does not test.

\section{Conclusion}

SCA uses neighboring rollout geometry to assign credit to GUI clicks. Residual
handles mixed-hit groups with a cross-fitted residual, while Prox supplies a
proximity signal for all-miss groups. In the reported offline evaluation, the
three-seed means are above the listed GUI-R1 point estimates on the grounding
and action-prediction suites. The synthetic study measures update direction and
error. Branch controls, online interaction, and broader backbones are left for
future evaluation.

\label{main:end}

\clearpage
\bibliographystyle{arxiv}
\bibliography{arxiv}

\appendix

\section{Experimental Details}
\label{app:reporting}

\subsection{Training configuration and comparison sources}
\label{app:training_details}

All SCA headline runs use Qwen2.5-VL-3B-Instruct with five responses per GUI state, the same EasyR1 clipped objective, prompt, evaluator, and optimizer family across seeds. The runtime correlation blend uses $[\rho_{\mathrm{lo}},\rho_{\mathrm{hi}}]=[0.3,0.7]$; the separate synthetic predictive-quality calibration is described in Appendix~\ref{app:gate_calibration}. The main-table SCA entries are three-seed summaries supplied with the manuscript; baseline entries are comparison point estimates reported by the cited works. Comparisons are descriptive because paired test-set uncertainty, matched training compute, and exact evaluator equivalence have not been established for every row. Sample standard deviations therefore describe SCA seed variation and do not claim statistical significance against published baselines.

\subsection{Metrics and aggregation}
\label{app:metrics}

Tables~\ref{tab:main_grounding} and~\ref{tab:main_cross_domain} define the
performance values used throughout the paper.  Type denotes action-type
accuracy, GR denotes click-grounding accuracy, and SR denotes step success.
The GUI-Odyssey column concerns predictions at recorded states, rather than
closed-loop task completion.

For ScreenSpot-Pro, the mean displayed in the text is the unweighted mean
of the twelve domain--UI-type scores in Table~\ref{tab:main_grounding}.
For ScreenSpot, Figure~\ref{fig:performance_overview} averages only the four
displayed Web/Desktop Text/Icon scores.
The Low-level Overall column is reproduced from Table~\ref{tab:main_cross_domain}.
All improvements in the text are absolute percentage-point differences,
calculated before rounding.  The sample SD in each SCA cell measures variation
across training seeds.  Baseline rows are comparison point estimates; the
reported differences are descriptive rather than paired significance tests.

\subsection{Recorded-case evaluator}
\label{app:case_metrics}

The case examples use a parsed action type and its applicable argument.
For a grounding action at reference $(x^*,y^*)$, the coordinate acceptance
test is $((x-x^*)/W)^2+((y-y^*)/H)^2<0.14^2$.
Text arguments use token-set F1 at least $0.5$; other actions use type match.
The examples preserve their original screenshot, instruction, predicted
coordinate, and evaluator outcome.  These tests specify the acceptance rule
used for the illustrated cases.

\section{Gate Calibration}
\label{app:gate_calibration}

The soft gate is calibrated independently of GUI training and test metrics.  The
candidate grid is
$q_{\mathrm{lo}}\in\{0.70,0.75,0.80\}$ crossed with
$q_{\mathrm{hi}}\in\{0.95,0.975,0.99\}$.  For each candidate, a paired
simulation draws $10{,}000$ groups from each of 17 pre-specified scenarios: five
Bernoulli nulls, clean linear-logit and box signals, and coordinate-, label-, and
combined-noise variants.  We calibrate the operating point at $N=5$ and retain
$N\in\{8,16\}$ as out-of-distribution stress settings.  All rate denominators
include every generated rollout, including predictor-invalid cases.

The selection rule imposes 14 one-sided constraints with Bonferroni family-wise
error rate $0.05$: for each Bernoulli null, the simultaneous 95\% upper bound on
gate activation is at most $0.10$ and on gate weight at least $0.9$ is at most
$0.02$; for each clean signal family, the simultaneous lower bounds on paired
activation lift and mean gate weight are at least $0.01$.  For each generated
signal group, the paired lift is its active-rollout fraction minus the expected
fraction under exact enumeration of all binary label assignments that preserve
that group's hit count and coordinates; the reported bound averages these
group-paired differences.  Among feasible bands,
the rule maximizes $q_{\mathrm{lo}}$ and then $q_{\mathrm{hi}}$, selecting
$[0.80,0.99]$ with seed 20260807.  Seeds 20260804--20260806 were used while
developing the synthetic scenarios and candidate grid.  The selected band was
then evaluated once, without reselection, on seed 20260808 and passed all
blocking constraints (Table~\ref{tab:gate_calibration}).  This holdout changes
the random draw within the same synthetic family, providing a repeatability check
for the selected band.

\begin{table}[!t]
\centering
\caption{Selection and one-shot holdout study for the fixed $N=5$ gate
(10,000 groups per scenario; selection seed 20260807; holdout seed 20260808).
``Null active'' and ``null strong'' are simultaneous 95\% upper bounds; ``clean
lift'' and ``clean mean weight'' are simultaneous lower bounds.}
\label{tab:gate_calibration}
\scriptsize
\begin{tabular}{lccc}
\toprule
Constraint summary & Required & Selection & Holdout \\
\midrule
Maximum null active-rollout rate & $\leq 0.1000$ & 0.0801 & 0.0793 \\
Maximum null rollout rate with $w\geq0.9$ & $\leq 0.0200$ & 0.0171 & 0.0155 \\
Minimum clean paired activation lift & $\geq 0.0100$ & 0.0139 & 0.0159 \\
Minimum clean mean gate weight & $\geq 0.0100$ & 0.0179 & 0.0189 \\
\bottomrule
\end{tabular}
\end{table}

\section{Cross-Fitting Properties and Worked Example}
\label{app:info_checks}

\subsection{Directional prediction and nested selection}
\label{app:directional_details}

For rollout $i$, let $I_{-i}=\{1,\ldots,N\}\setminus\{i\}$.
The four unit directions are $(1,0)$, $(0,1)$,
$(1,1)/\sqrt{2}$, and $(1,-1)/\sqrt{2}$.
The projected coordinate is $u_{j,d}=e_d^\top a_j$.
Fitting on $I_{-i}$ gives $\hat r_{i,d}=f_{i,d}(u_{i,d})$.
In the inner loop, a fit on $I_{-i}\setminus\{j\}$ gives
$\tilde r_{i,j,d}$, with null prediction
$m_{i,j}=(N-2)^{-1}\sum_{k\in I_{-i}\setminus\{j\}}r_k$.
For finite predictions with positive prediction--reward correlation and
nonzero null error, the skill is
\begin{gather}
s_{i,d}=\left[
1-\frac{\sum_{j\in I_{-i}}(r_j-\tilde r_{i,j,d})^2}
        {\sum_{j\in I_{-i}}(r_j-m_{i,j})^2}
\right]_+,
\label{eq:candidate_skill}
\end{gather}
and it is zero otherwise.  Positive skills give weights
$\beta_{i,d}=s_{i,d}/\sum_{d'}s_{i,d'}$.  These weights combine both the
outer predictions and the inner validation traces:
$\hat r_i=\sum_d\beta_{i,d}\hat r_{i,d}$ and
$\tilde r_{i,j}=\sum_d\beta_{i,d}\tilde r_{i,j,d}$.
Applying the same skill calculation to the combined trace gives $q_i$.
The gate is $w_i=\clip((q_i-0.80)/0.19,0,1)$.
Binary predictions are clipped to $[-1,2]$ before scoring or residualization.
If no candidate has positive skill, rollout $i$ has no valid prediction and
receives zero spatial weight.  Appendix~\ref{app:safeguards} details the
remaining numerical cases.

\subsection{Properties and worked example}
\label{app:spatial_inputs}

\begin{proposition}[Held-out reward separation]
\label{prop:self_label}
Consider a mixed-hit binary group with complete coordinates.  For rollout $i$,
fix the sampled actions and rewards $r_{-i}$.  The cross-fitted prediction
$\hat r_i$, candidate weights $\beta_{i,d}$, ensemble skill $q_i$, and gate $w_i$
are invariant to changes in $r_i$ that preserve the mixed-hit route.
\end{proposition}
\begin{proof}
The outer candidates use $I_{-i}$, and every nested candidate uses a subset of
$I_{-i}$.  Candidate skill, aggregation weights, the ensemble trace, and the
gate are therefore functions of $a_{1:N}$ and $r_{-i}$.  The held-out reward
is used later to form residual and group-relative credit; it is absent from
its own prediction and gate under the fixed-route condition.
\end{proof}

For the sampled token sequence $Y_i$, let
$S_i=\nabla_\theta\log\pi_\theta(Y_i\mid x)$ and let
$H_i=H_i(a_{1:N},r_{-i};x)$ denote its cross-fitted spatial prediction.  The
simplified residual update satisfies
\begin{equation}
\E\!\left[\frac{1}{N}\sum_i S_i(r_i-H_i)\right]
=g_R-\E\!\left[\frac{1}{N}\sum_i S_iH_i\right],
\label{eq:bias_decomposition}
\end{equation}
where $g_R=\nabla_\theta\E[r(Y,x)]$.  The second term is the spatial correction
introduced by action-conditioned credit.  Since $H_i$ is evaluated at the
sampled action, excluding $r_i$
does not imply $\E[S_iH_i]=0$.  Section~\ref{sec:gradient_audit} measures the
resulting direction, scale, and error in the synthetic setting.

\noindent\textbf{Inputs to spatial credit.}
Table~\ref{tab:spatial_inputs} distinguishes the information used by each
rule.  Residual's feature map is $e_d^\top a_j$.  Given the sampled actions,
validity masks, and supplied response and routing channels, its calculation
does not read a target box, centroid, or diagonal.  Those channels may already
encode target supervision through the evaluator.  Prox and distance-based
rules additionally use target geometry to construct their spatial signal.

\begin{table}[H]
\centering
\caption{\textbf{Information used to construct credit.} The task evaluator
supplies scalar rewards; spatial rules use the additional quantities listed.}
\label{tab:spatial_inputs}
\small
\begin{tabular*}{\textwidth}{@{\extracolsep{\fill}}lp{0.47\textwidth}l@{}}
\toprule
Rule & Spatial reference & Target geometry\\
\midrule
Binary GRPO & None; normalize binary group rewards & In the evaluator\\
Distance shaping & Fixed function of action-to-target distance & Read explicitly\\
Distance LOO & Reward trend fitted against target distance & Read explicitly\\
SCA-Residual & Reward trend fitted against sampled coordinates & In the evaluator\\
SCA-Prox & Exponential proximity, scaled after normalization & Read explicitly\\
\bottomrule
\end{tabular*}
\end{table}

\noindent\textbf{Illustrative mixed group.}
Table~\ref{tab:worked_credit} works through five sampled actions.  GRPO
assigns the same credit to all three misses, whereas Residual uses their
different fitted references.  The first miss has positive final credit:
its observed reward $0$ exceeds the extrapolated reference $-0.423$.
The reference is a clipped linear prediction, not a success probability.
The example illustrates relative residual credit, including its ability to
assign a miss more credit than a hit in the same group.

\begin{table}[H]
\centering
\caption{\textbf{A deterministic five-click calculation.} Values are rounded
to three decimals.  A dash denotes an invalid prediction, with zero gate weight.}
\label{tab:worked_credit}
\small
\begin{tabular*}{\textwidth}{@{\extracolsep{\fill}}lrrrrr@{}}
\toprule
Click $a_i$ & Reward $r_i$ & GRPO & Prediction $\hat r_i$ & Gate $w_i$ & Residual\\
\midrule
$(0,0)$ &0&$-0.730$&$-0.423$&1.000&0.770\\
$(1,2)$ &0&$-0.730$&0.309&0.190&$-1.095$\\
$(2,1)$ &0&$-0.730$&0.309&0.190&$-1.095$\\
$(4,5)$ &1&1.095&--&0.000&0.710\\
$(5,4)$ &1&1.095&--&0.000&0.710\\
\bottomrule
\end{tabular*}
\end{table}

\section{Algorithm}
\label{app:algorithm}

With response-token mask $M_{it}$ and likelihood ratio
$\rho_{it}=\pi_\theta(y_{it}\mid y_{i,<t},x)/
\pi_{\theta_{\mathrm{old}}}(y_{it}\mid y_{i,<t},x)$, the SCA actor uses
\begin{equation}
\mathcal L_{\mathrm{actor}}=-\frac{\sum_{i,t}M_{it}
\min\!\left\{\rho_{it}\sg(A_i),
\clip(\rho_{it},1-\varepsilon_c,1+\varepsilon_c)\sg(A_i)\right\}}
{\sum_{i,t}M_{it}}
+\lambda_{\mathrm{KL}}\widehat D_{\mathrm{KL}},
\label{eq:actor_objective}
\end{equation}
where $\varepsilon_c=0.2$ and $\lambda_{\mathrm{KL}}=10^{-2}$.  The implemented
low-variance KL term is the masked mean of
$\clip(\exp(\delta_{it})-\delta_{it}-1,-10,10)$ with
$\delta_{it}=\log\pi_{\mathrm{ref}}-\log\pi_\theta$.

\begin{algorithm}[!t]
\small
\caption{Strict \scaresidual credit for one binary-reward group}
\label{alg:core}
\begin{algorithmic}[1]
\Require $G=\{(a_i,r_i)\}_{i=1}^N$, directions $\{e_d\}_{d=1}^D$,
gate band $[q_{\mathrm{lo}},q_{\mathrm{hi}}]$
\State $A^{\mathrm{GRPO}}\gets z(r)$
\If{all rewards are equal or any rollout coordinate is invalid}
  \State \Return $A^{\mathrm{GRPO}}$
\EndIf
\For{$i=1,\ldots,N$}
  \State $I_{-i}\gets\{1,\ldots,N\}\setminus\{i\}$
  \For{direction $d$}
    \State Fit $f_{i,d}$ on $I_{-i}$ and predict $\hat r_{i,d}$
    \For{$j\in I_{-i}$}
      \State Fit on $I_{-i}\setminus\{j\}$ and predict $\tilde r_{i,j,d}$
      \State $m_{i,j}\gets\operatorname{mean}\{r_k:k\in I_{-i}\setminus\{j\}\}$
    \EndFor
    \State Clip binary outer and nested predictions to $[-1,2]$
    \State Compute signed positive skill $s_{i,d}$ using Eq.~\eqref{eq:candidate_skill}
  \EndFor
  \If{at least one candidate has positive skill}
    \State Skill-weight candidates to obtain $\hat r_i$ and nested predictions
    \State Score aggregate nested predictions to obtain $q_i$
    \If{$\hat r_i$ is finite and $q_i>0$}
      \State $w_i\gets\clip((q_i-q_{\mathrm{lo}})/(q_{\mathrm{hi}}-q_{\mathrm{lo}}),0,1)$
    \Else
      \State Mark $i$ prediction-invalid; set $\hat r_i\gets\mathrm{NaN}$ and $w_i\gets0$
    \EndIf
  \Else
    \State Mark $i$ prediction-invalid; set $\hat r_i\gets\mathrm{NaN}$ and $w_i\gets0$
  \EndIf
\EndFor
\If{$\max_i w_i=0$}
  \State \Return $A^{\mathrm{GRPO}}$
\EndIf
\State $V\gets\{i:\hat r_i\text{ is finite}\}$
\State $A_i^{\mathrm{res}}\gets z_V(r_V-\hat r_V)_i$ for $i\in V$; $0$ otherwise
\State \Return $z(w\odot A^{\mathrm{res}}+(1-w)\odot A^{\mathrm{GRPO}})$
\end{algorithmic}
\end{algorithm}

\subsection{Numerical implementation}
\label{app:safeguards}

Degenerate projected pairs are omitted from median-slope fitting.  A direction
with no finite candidate prediction receives zero skill.  Binary outer and nested
predictions are clipped to $[-1,2]$ before residualization and gate scoring.  A non-positive prediction--reward
correlation, non-finite statistic, degenerate null SSE, or non-positive skill sets
the candidate quality to zero.  In the gated mixed-hit Residual path, if no candidate
remains or any coordinate in the calibrated $N=5$ group is invalid, the group
returns ordinary group-relative credit.  The all-miss Prox path instead operates
on its nonempty valid-coordinate subset as defined in Eq.~\eqref{eq:prox}.  All standardizations use a stabilized sample standard deviation and
return zero on a degenerate vector.  The training launcher sets the gate band
to $[0.80,0.99]$.

\noindent\textbf{Fit count.}
For $D$ directions and $N$ responses, Residual performs $DN$ outer fits and
$DN(N-1)$ inner fits.  With $D=4$ and $N=5$, this is 20 four-point fits and
80 three-point fits per mixed-hit group.  Each outer fit considers at most
six coordinate pairs and each inner fit at most three.  This work uses the
completed rollout group; it requires no additional policy generations and
no fitted predictor at inference.
\FloatBarrier

\section{Exact-Gradient Analysis}
\label{sec:gradient_audit}
\label{app:gradient_details}

\subsection{Synthetic setting}

The audit uses a diagonal Gaussian policy
$a\mid x_s\sim\mathcal N(\mu,I_2)$ with shared mean parameter
$\mu=(0,0)$ and an axis-aligned target box
$T_s=[c_{s,1}-h_{s,1},c_{s,1}+h_{s,1}]\times
[c_{s,2}-h_{s,2},c_{s,2}+h_{s,2}]$.  The objective is the exact binary return
$J_s(\mu)=\Pr(a\in T_s)$.  For general policy standard deviation $\sigma$, let
$L_k=(c_{s,k}-h_{s,k}-\mu_k)/\sigma$,
$U_k=(c_{s,k}+h_{s,k}-\mu_k)/\sigma$,
$p_k=\Phi(U_k)-\Phi(L_k)$, and
$d_k=[\phi(L_k)-\phi(U_k)]/\sigma$.  Then
$J_s=p_1p_2$ and $\nabla_\mu J_s=(d_1p_2,p_1d_2)$, which is the exact
reference used in Table~\ref{tab:bandit_scenarios}.

Table~\ref{tab:bandit_scenarios} gives every context.  The audit uses
seed 20260803, $\sigma=1$, $N=5$, and $100{,}000$ independently sampled groups
per context.  All estimators receive the same sampled actions.  The uniform
contextual mixture pools $400{,}000$ group gradients and compares their mean with
the arithmetic mean of the four exact context gradients.  These synthetic
mixture outcomes support the estimator analysis.

\begin{table}[!t]
\centering
\caption{Exact-gradient contextual-bandit scenarios.  $c$ is target center, $h$
is target half-extent, and $P(\mathrm{hit})$ is analytic at $\mu=(0,0)$.}
\label{tab:bandit_scenarios}
\scriptsize
\begin{tabular}{lccccrcc}
\toprule
Scenario & $c_x$ & $c_y$ & $h_x$ & $h_y$ & $P(\mathrm{hit})$ & $\nabla_xJ$ & $\nabla_yJ$ \\
\midrule
Dense offset & 0.35 & $-0.20$ & 1.20 & 1.20 & 0.56418 & 0.12015 & $-0.06842$ \\
Balanced offset & 0.45 & 0.25 & 0.85 & 0.70 & 0.28076 & 0.09896 & 0.05948 \\
Sparse offset & 0.90 & 0.55 & 0.55 & 0.45 & 0.08733 & 0.07110 & 0.04489 \\
Rare/far & 1.45 & $-0.80$ & 0.45 & 0.35 & 0.02615 & 0.03550 & $-0.02009$ \\
\bottomrule
\end{tabular}
\end{table}

\begin{table}[!t]
\centering
\caption{Estimated mean gradient and component-wise Monte Carlo standard error
on the uniform contextual mixture.  The exact row has no sampling error.}
\label{tab:gradient_mcse}
\scriptsize
\begin{tabular}{lcc}
\toprule
Estimator & Mean $g_x\;\pm\;$MCSE & Mean $g_y\;\pm\;$MCSE \\
\midrule
Exact reference & 0.08143 & 0.00396 \\
Binary \grpo & $0.13220\pm0.00045$ & $0.00770\pm0.00044$ \\
2-D OLS-LOO & $-0.00307\pm0.00034$ & $-0.00489\pm0.00034$ \\
2-D Ridge-LOO & $-0.00304\pm0.00034$ & $-0.00490\pm0.00034$ \\
\scaresidual & $0.12396\pm0.00043$ & $0.00818\pm0.00042$ \\
Target-distance LOO & $-0.01560\pm0.00045$ & $-0.00416\pm0.00045$ \\
Full \sca & $0.12995\pm0.00043$ & $0.00738\pm0.00042$ \\
\scaresidual without final gate & $0.09891\pm0.00041$ & $0.00909\pm0.00040$ \\
\bottomrule
\end{tabular}
\end{table}

\subsection{Gradient error}
\label{app:gradient_error}

For a sampled response $Y_i$, let
$S_i=\nabla_\theta\log\pi_\theta(Y_i\mid x)$.
Before PPO clipping, the on-policy group update is
$\hat g=N^{-1}\sum_i S_i A_i$.
On the same sampled group, Residual and GRPO therefore satisfy
$\hat g_C=\hat g_G+X$, where
$X=N^{-1}\sum_i S_i(A_i^{\mathrm{Residual}}-A_i^{\mathrm{GRPO}})$.
Writing $g_R$ for the exact return gradient and expanding the squared errors
gives
\begin{gather}
\E\lVert\hat g_C-g_R\rVert_2^2-\E\lVert\hat g_G-g_R\rVert_2^2
=\E\lVert X\rVert_2^2+2\E\langle\hat g_G-g_R,X\rangle.
\label{eq:paired_mse}
\end{gather}
Thus gradient MSE decreases when the correction opposes the baseline error
strongly enough to offset its own squared magnitude.  This condition concerns
the score-weighted, normalized update, whereas reward fitting concerns
$r_i-\hat r_i$.  Section~\ref{sec:visualization} measures the former directly
against the exact gradient in the synthetic mixture.

\subsection{Controlled corruption}
\label{app:corruption}

We directly perturb one reward or one coordinate in $50{,}000$ synthetic groups
per scenario, while measuring the change on the other rollouts.  Table~\ref{tab:corruption}
reports gradient drift and the change in exact-gradient MSE.  The two estimators
are close in reward-flip settings; OLS is stronger for the balanced coordinate
spike, while the median ensemble has a slightly smaller MSE increase for the
sparse coordinate spike.  The preferred fit therefore depends on the
perturbation in these small-group settings.

\begin{table}[H]
\centering
\caption{Controlled corruption study; smaller is better for both metrics.}
\label{tab:corruption}
\scriptsize
\begin{tabular}{llrrr}
\toprule
Context & Corruption & Estimator & Drift & $10^3$ MSE $\Delta$ \\
\midrule
Balanced & reward flip & median ensemble & 0.146 & $-0.36$ \\
Balanced & reward flip & OLS ensemble & 0.145 & $-1.32$ \\
Balanced & coordinate spike & median ensemble & 0.042 & 9.24 \\
Balanced & coordinate spike & OLS ensemble & 0.038 & 8.21 \\
Sparse & reward flip & median ensemble & 0.021 & $-2.42$ \\
Sparse & reward flip & OLS ensemble & 0.022 & $-2.61$ \\
Sparse & coordinate spike & median ensemble & 0.009 & 2.35 \\
Sparse & coordinate spike & OLS ensemble & 0.009 & 2.40 \\
\bottomrule
\end{tabular}
\end{table}

\section{Prompt Templates and Recorded Cases}
\label{app:cases}

This appendix makes the evaluation interface concrete.  The prompt follows the
GUI-R1 instruction format and is shared across the grounding and action-prediction
benchmarks; only the admissible action set changes with the benchmark.  Each
recorded case below is drawn from the released evaluation replay and keeps the
original screenshot, instruction, coordinates, and evaluator outcome.

\subsection{Prompt templates}

The ScreenSpot and ScreenSpot-Pro runs use the following user message, where
\texttt{\{instruction\}} is the benchmark instruction and
\texttt{\{history\}} is the action history (\texttt{None} for a single-step
case):

\begin{scaprompt}{Grounding Prompt}
You are RUN1-R1, a reasoning GUI Agent Assistant. In this UI screenshot
\textless image\textgreater, I want you to continue executing the command
'\{instruction\}', with the action history being '\{history\}'.
Please provide the action to perform (enumerate from ['click']), the point where
the cursor is moved to (integer) if a click is performed, and any input text
required to complete the action.
Output the thinking process in \textless think\textgreater
\textless/think\textgreater tags, and the final answer in
\textless answer\textgreater \textless/answer\textgreater tags as follows:
\textless think\textgreater ... \textless/think\textgreater
\textless answer\textgreater[\{'action': 'click', 'point': [x, y],
'input\_text': 'no input text [default]'\}]\textless/answer\textgreater
\end{scaprompt}

OmniAct uses the same message skeleton with a benchmark-specific action set:

\begin{scaprompt}{Unified Action Prompt}
You are GUI-R1, a reasoning GUI Agent Assistant. In this UI screenshot
\textless image\textgreater, I want you to continue executing the command
'\{instruction\}', with the action history being '\{history\}'.
Please provide the action to perform (enumerate from
\{action\_set\}), the point where the cursor is moved to (integer) if
a click is performed, and any input text required to complete the action.
Output the thinking process in \textless think\textgreater
\textless/think\textgreater tags, and the final answer in
\textless answer\textgreater \textless/answer\textgreater tags as follows:
\textless think\textgreater ... \textless/think\textgreater
\textless answer\textgreater[\{'action': enum[...], 'point': [x, y],
'input\_text': 'no input text [default]'\}]\textless/answer\textgreater
\end{scaprompt}

\begin{scaprompt}{Benchmark Action Vocabularies}
\rmfamily
\begin{tabular}{@{}l>{\raggedright\arraybackslash}p{0.62\linewidth}@{}}
\toprule
Benchmark & Action vocabulary\\
\midrule
ScreenSpot / ScreenSpot-Pro & \texttt{click}\\
OmniAct-Web & \texttt{press\_tab, moveto, rightclick, press\_enter, scroll,}\\
& \texttt{click, press\_down, hotkey, press\_space, doubleclick}\\
OmniAct-Desktop & \texttt{click, moveto, press\_pgdn}\\
\bottomrule
\end{tabular}
\end{scaprompt}

The grounding script uses the identifier `RUN1-R1`; the OmniAct scripts use
`GUI-R1`.  For non-coordinate actions, the parser uses the sentinel point
\texttt{[-100,-100]}; pointer actions use the benchmark screenshot coordinate
frame.  SCA receives the parsed coordinates and scalar rewards through this
message/evaluator interface.

\subsection{Verifiable action fields}

Each answer is parsed as
$o_i=(o_i^{\mathrm{act}},o_i^{\mathrm{point}},o_i^{\mathrm{text}})$.  Evaluation
checks the action type and its applicable argument: an accepted coordinate for
pointer actions, token-set F1 $\geq0.5$ for text, and type match otherwise.  We
set $h_i=\Ind[o_i^{\mathrm{act}}=g^{\mathrm{act}}]
\Ind[\operatorname{ArgOK}(o_i,g)]$, with $h_i=0$ for malformed outputs.  Binary
runs use $r_i=h_i$; shaped runs use $h_i$ for SCA routing and the shaped scalar
for policy credit.

\clearpage
\subsection{Recorded evaluation cases}

\begin{scacase}{Case 1: Web Shopping Page}
\textbf{Source:} OmniAct-Web, output index 0, dataset row 0.\\
\textbf{Task:} examine the items that have been added to the shopping trolley
so far.\\
\textbf{History:} None.\\
\textbf{Parsed SCA output:} action = click; point = $(1276,114)$.\\
\textbf{Ground truth:} action = click; point = $(1287,115)$.\\
\textbf{Evaluation:} Correct action and accepted coordinate.
\tcblower
\begin{center}
\begin{tikzpicture}
\node[anchor=south west,inner sep=0] (caseimage) at (0,0)
  {\includegraphics[width=0.96\linewidth]{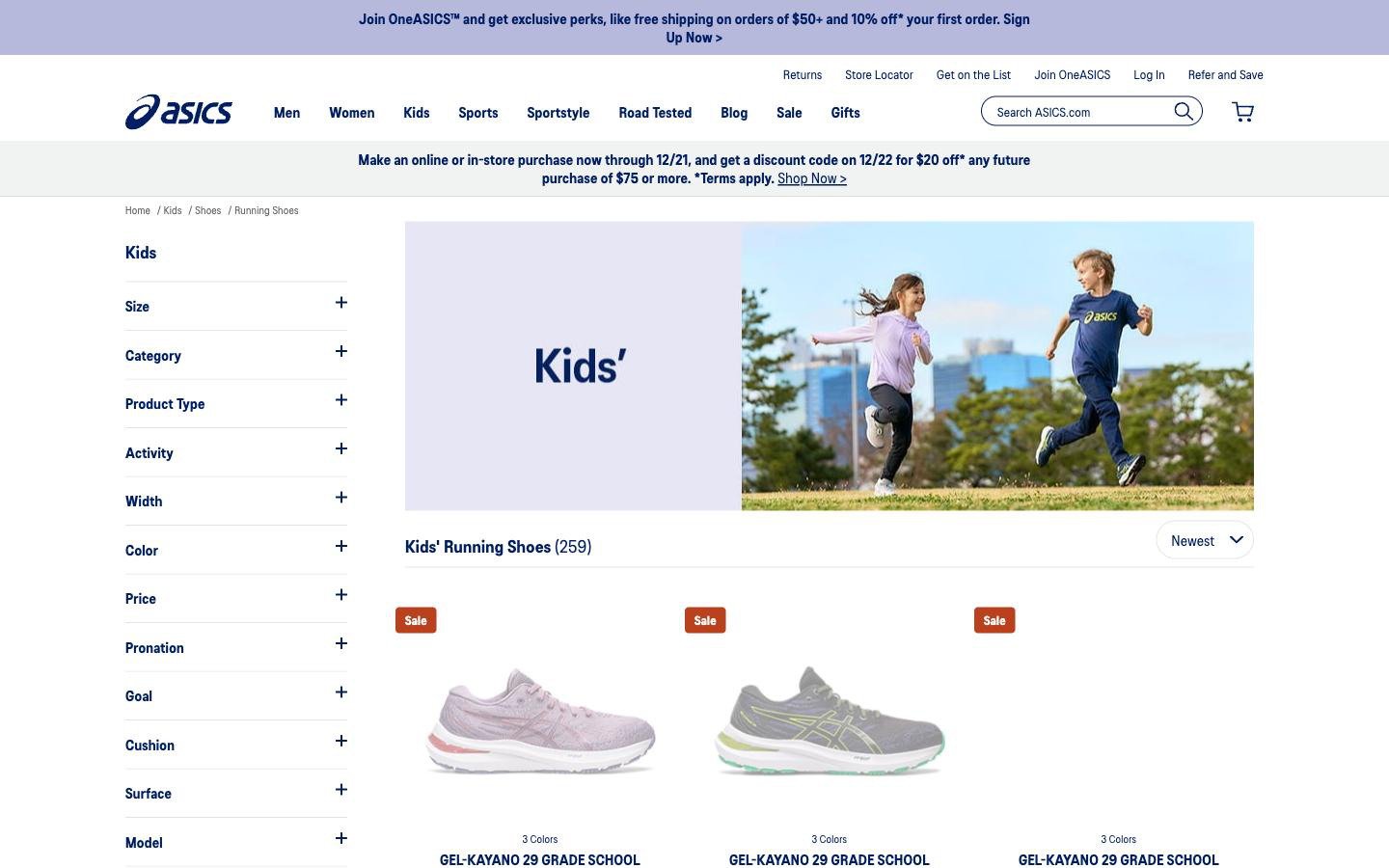}};
\begin{scope}[x={(caseimage.south east)},y={(caseimage.north west)}]
  \draw[guirorange,line width=0.9pt] (0.894,0.872)
    ellipse[x radius=0.012,y radius=0.019];
  \draw[scablue,line width=0.9pt] (0.886,0.873) ++(-0.010,0)--++(0.020,0);
  \draw[scablue,line width=0.9pt] (0.886,0.873) ++(0,-0.016)--++(0,0.032);
\end{scope}
\end{tikzpicture}
\end{center}
\centering\footnotesize Orange ellipse: ground-truth target. Blue cross: parsed SCA coordinate.
\end{scacase}

\clearpage
\begin{scacase}{Case 2: Desktop Map}
\textbf{Source:} OmniAct-Desktop, output index 8, dataset row 16.\\
\textbf{Task:} change the current map display to satellite view.\\
\textbf{History:} None.\\
\textbf{Parsed SCA output:} action = click; point = $(2465,37)$.\\
\textbf{Ground truth:} action = click; point = $(2413,39)$.\\
\textbf{Evaluation:} Correct action and accepted coordinate.
\tcblower
\begin{center}
\begin{tikzpicture}
\node[anchor=south west,inner sep=0] (caseimage) at (0,0)
  {\includegraphics[width=\linewidth]{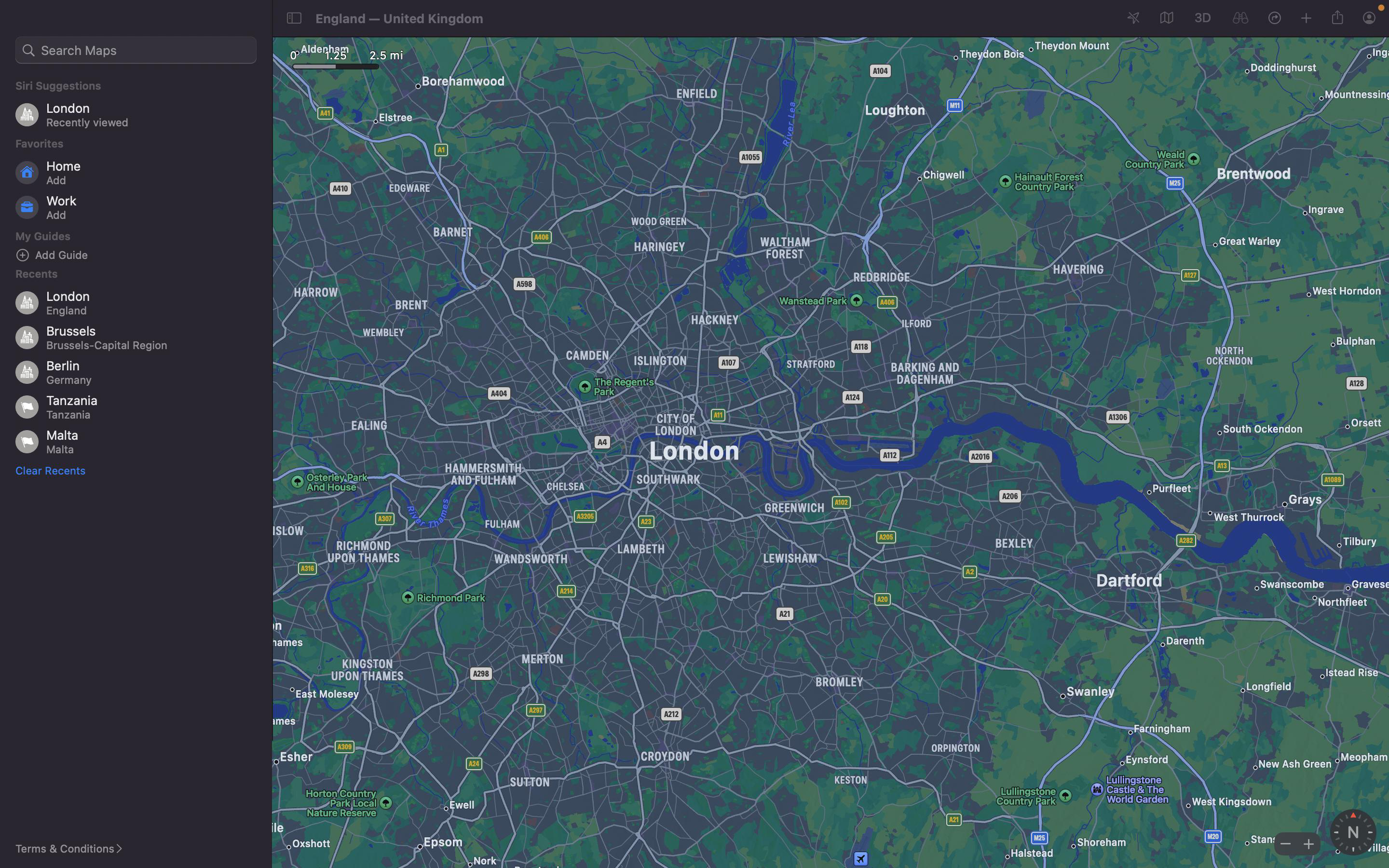}};
\begin{scope}[x={(caseimage.south east)},y={(caseimage.north west)}]
  \draw[guirorange,line width=0.9pt] (0.838,0.978)
    ellipse[x radius=0.010,y radius=0.016];
  \draw[scablue,line width=0.9pt] (0.856,0.979) ++(-0.009,0)--++(0.018,0);
  \draw[scablue,line width=0.9pt] (0.856,0.979) ++(0,-0.014)--++(0,0.028);
\end{scope}
\end{tikzpicture}
\end{center}
\centering\footnotesize Orange ellipse: ground-truth target. Blue cross: parsed SCA coordinate.
\end{scacase}

\clearpage
\begin{scacase}{Case 3: Web Double-Click Action}
\textbf{Source:} OmniAct-Web, output index 546, dataset row 1092.\\
\textbf{Task:} execute double click function on
\texttt{search\_store\_locations} tab.\\
\textbf{History:} None.\\
\textbf{Parsed SCA output:} action = doubleclick; point = $(1205,14)$.\\
\textbf{Ground truth:} action = doubleclick; point = $(1221,14)$.\\
\textbf{Evaluation:} Correct action and accepted coordinate.
\tcblower
\begin{center}
\begin{tikzpicture}
\node[anchor=south west,inner sep=0] (caseimage) at (0,0)
  {\includegraphics[width=0.96\linewidth]{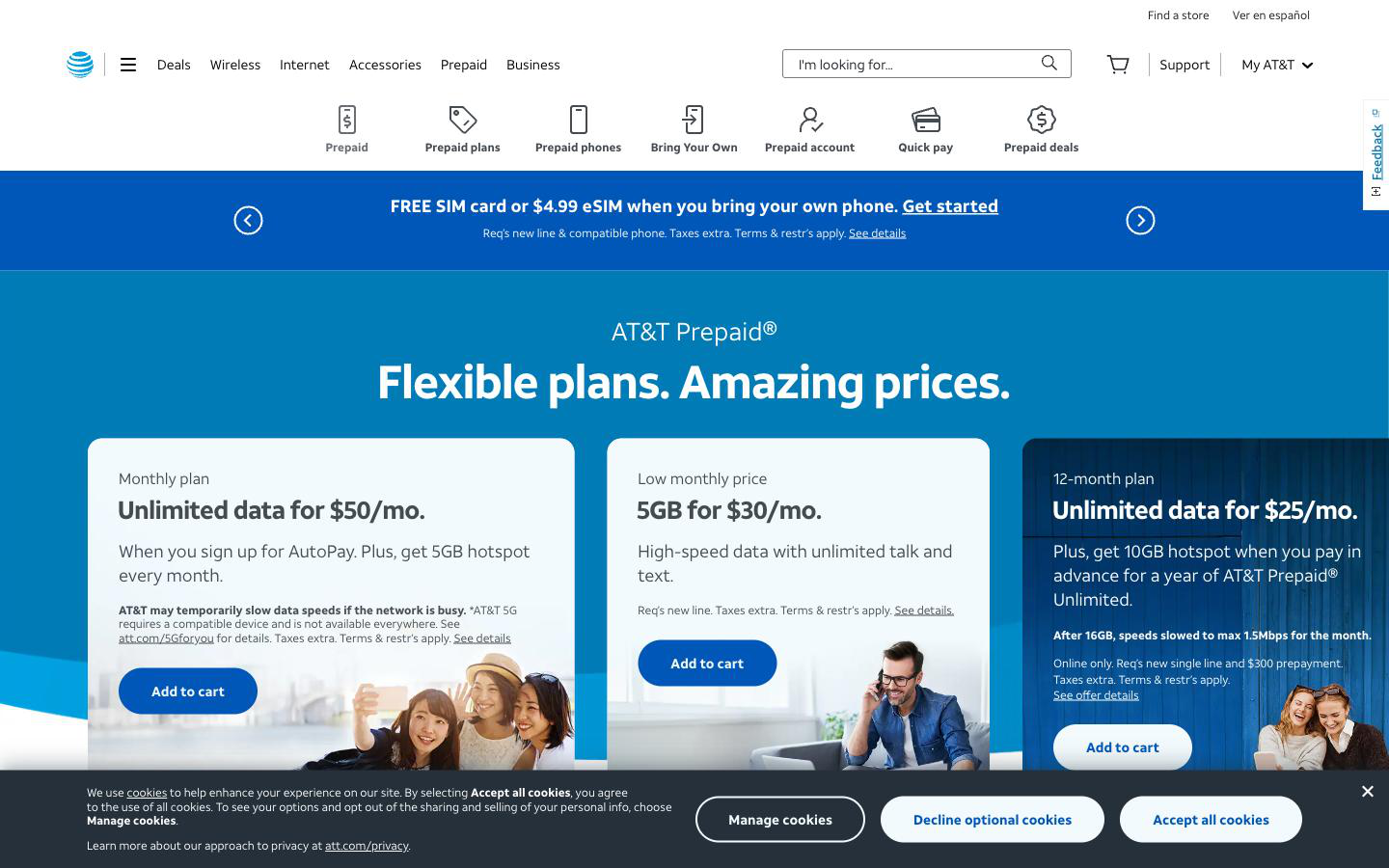}};
\begin{scope}[x={(caseimage.south east)},y={(caseimage.north west)}]
  \draw[guirorange,line width=0.9pt] (0.848,0.984)
    ellipse[x radius=0.012,y radius=0.019];
  \draw[scablue,line width=0.9pt] (0.837,0.984) ++(-0.010,0)--++(0.020,0);
  \draw[scablue,line width=0.9pt] (0.837,0.984) ++(0,-0.016)--++(0,0.032);
\end{scope}
\end{tikzpicture}
\end{center}
\centering\footnotesize Orange ellipse: ground-truth target. Blue cross: parsed SCA coordinate.
\end{scacase}

\noindent\textbf{Case summary.}
The three cases cover a standard web click, a high-resolution desktop click, and
a double-click action.  They illustrate the prediction and evaluation interface;
the rollout-level credit construction is shown separately by the deterministic
mixed-group example in Appendix~\ref{app:info_checks}.

\end{document}